%% file: learning.Sols.neurips_2019.withappendix.tex
\documentclass{article}

\PassOptionsToPackage{numbers, compress}{natbib}

 \usepackage[preprint]{neurips_2019}

\usepackage[utf8]{inputenc} 
\usepackage[T1]{fontenc}    
\usepackage{url}            
\usepackage{booktabs}       
\usepackage{amsfonts}       
\usepackage{nicefrac}       
\usepackage{microtype}      

\usepackage{amsmath}
\usepackage{amsthm}
\usepackage{thmtools}
\usepackage{amssymb}
\usepackage{bbm}
\usepackage{graphicx}
\usepackage{xspace}
\usepackage[colorinlistoftodos]{todonotes}
\usepackage{tabularx}

\usepackage{appendix}

\usepackage{paralist}
\usepackage{comment}
\usepackage{tcolorbox}
\specialcomment{yaircomment}{\begin{tcolorbox}[colback=red!5!white,colframe=red,title=Yair:]\begingroup\ttfamily}{\endgroup\end{tcolorbox}}
\specialcomment{tushantcomment}{\begin{tcolorbox}[colback=blue!5!white,colframe=blue,title=Tushant:]\begingroup\ttfamily}{\endgroup\end{tcolorbox}}
\excludecomment{yaircomment}
\excludecomment{tushantcomment}

\declaretheorem[name=Theorem,numberwithin=section]{theorem}
\newtheorem{lemma}[theorem]{Lemma}

\theoremstyle{definition}

\declaretheorem[name=Definition,sibling=theorem]{definition}
\newtheorem{example}[theorem]{Example}

\newcommand{\tup}[1]{\langle #1 \rangle}
\newcommand{\R}{\mathbb{R}}
\newcommand{\E}{\mathbb{E}}
\newcommand{\core}{\mathit{core}}
\renewcommand{\cal}[1]{\mathcal{#1}}
\newcommand{\eps}{\varepsilon}

\newcommand{\CE}{{\mathit{CE}}}
\newcommand{\Cond}{{\mathit{Cond}}}
\newcommand{\probleminstance}{{\Psi}}
\newcommand{\probleminstancename}{\textsc{StatisticalSolution}\xspace}

\newcommand{\G}{\mathbb{G}}
\newcommand{\HC}{\mathbb{H}}
\newcommand{\I}{\mathbb{I}}
\newcommand{\Sol}{\mathbb{S}}
\DeclareMathOperator{\argmax}{\mathrm{argmax}}
\DeclareMathOperator{\argmin}{\mathrm{argmin}}
\newcommand{\VC}{\mathit{VC}}
\newcommand{\Ind}{\mathbbm{1}}
\newcommand{\Sd}{\mathit{Sd}}

\newcommand{\UC}{\mathit{UC}}
\newcommand{\PAC}{\mathit{PAC}}

\title{A Learning Framework for Distribution-Based Game-Theoretic Solution Concepts}

\author{%
Tushant Jha\\
International Institute of Information Technology, Hyderabad\\
\texttt{particle.mania@gmail.com} \\
\And
Yair Zick\\
National University of Singapore, Singapore\\
\texttt{zick@comp.nus.edu.sg}
}

\begin{document}

\maketitle

\begin{abstract}
	The past few years have seen several works on learning economic solutions from data; these include optimal auction design, function optimization, stable payoffs in cooperative games and more. 
	In this work, we provide a unified learning-theoretic methodology for modeling such problems, and establish tools for determining whether a given economic solution concept can be learned from data. 
	Our learning theoretic framework generalizes a notion of function space dimension --- the graph dimension --- adapting it to the solution concept learning domain. 
	We identify sufficient conditions for the PAC learnability of solution concepts, and show that results in existing works can be immediately derived using our methodology. Finally, we apply our methods in other economic domains, yielding a novel notion of PAC competitive equilibrium and PAC Condorcet winners. 
\end{abstract}
\input{intro}
\input{primer}
\input{gametheoryprimer}
\input{sol-concepts}

\input{vctheory2}
\input{concretizations}

\input{conclusion}

\bibliographystyle{plainnat}
\bibliography{abb,statsolconcepts}

\newpage
\appendix
\renewcommand{\appendixpagename}{Appendix}
\appendixpage
\input{apdxvctheory}

\input{apdxcoopgames}
\input{apdxmarkets}
\input{apdxcondorcet}
\end{document}

%% file: intro.tex
\section{Introduction}\label{sec:intro}
Recent years have seen widespread application of learning-theoretic notions in economic domains. 
Rather than assuming full knowledge of the underlying domain (or a prior over the domain space), one assumes access to a dataset of past instances, and employs learning-theoretic tools in order to obtain approximate solutions. 
Consider the following simple example of learning a solution from data: we wish to find the maximum of a function $f:\R^n \to \R$; we do not have access to $f$, but rather to a dataset of the form $\tup{\vec x_1,f(\vec x_1)},\dots,\tup{\vec x_m,f(\vec x_m)}$. One way to find a likely candidate point would be to use classic learning-theoretic tools \cite{anthony1999learning}, learn an approximation $f^*$ of $f$, and compute the maximum of $f^*$; however, this goes above and beyond the problem requirement: the approximability of $f$ depends on its {\em hypothesis class} (whether $f$ is a linear function, a two-layer neural network etc.), and on the approximation robustness. A much simpler solution is available: if the number of samples is sufficiently large, taking the {\em empirical maximum} of $f$ over the dataset --- i.e. $\vec x^* \in \argmax_j \{f(\vec x_j):j \in [m]\}$ --- yields a point that is likely to be greater in value than any future point sampled from the same distribution as the original dataset. 

The same reasoning applies to other economic solutions: a naive approach to inferring solutions from data would be to learn an approximate model (e.g. learn a function $f^*$ which approximates $f$ in the maximization example above), and then try generating solutions for the approximate model. However, as has been shown in the literature, learning an approximate model may: 
\begin{enumerate}
	\item be insufficient for generating `good' solutions (this is the case in \cite{sliwinski2017hedonic}) 
	\item require an exponential number of samples, whereas directly learning solutions is easy. Indeed, finding a payoff in the core of TU games is easy \cite{balcan2015learning-arxiv,balkanski2017cost}, while PAC learning cooperative games requires an exponential number of samples \cite{balcan2015learning}; this is also the case for finding an empirical maximum in the example above.
\end{enumerate}
Recent works directly learn solutions to classic optimization problems, as well as solutions in game-theoretic domains. These lines of work have progressed more or less independently, proving that solutions in a specific problem domain can (or cannot) be efficiently inferred from data; however, there has been no attempt to provide a unified theory of learning solution concepts from data. This is where our work comes in.  
\subsection{Our Contributions}\label{sec:contrib}
We begin by establishing a learning-theoretic framework for learning solution concepts from data. Unlike classic learning problem spaces, solution concepts do not inhabit the same space as the observed samples (e.g. when learning an approximate maximum, the function space is $\R^n \to \R$, whereas the solution space is $\R^n$). In Section \ref{sec:pac-sol-concepts}, we define the {\em solution dimension}: this quantity depends on both the hypothesis class of the underlying game and the solution space. The solution dimension generalizes the graph dimension \cite{daniely2011multiclass} in PAC learning, and serves a similar purpose: if the solution dimension is low, then a distribution-based solution can be efficiently learned from samples. Drawing on notions of shattering from VC dimension, we introduce {\em solution concept shattering} which is used to bound the solution dimension in various domains. 
We also show that the existence of a consistent solution and a low graph dimension are sufficient conditions for PAC learning solutions, simplifying technical learnability arguments in existing works, as well as paving the way for a straightforward learnability approach of other solution concepts. 
In Section~\ref{sec:comp-solutions}, we apply our methodology to immediately derive sample complexity bounds on learning solutions for hedonic games, as well as for two novel domains: market equilibria, and Condorcet winners in voting. 
\subsection{Related Work}\label{sec:related}
Several recent works study learning solutions from data; these include solutions in cooperative games \cite{balcan2015learning,balkanski2017cost,igarashi2019learning,sliwinski2017hedonic}, combinatorial auctions \cite{balcan2018general, balcan2016sample,Brero2018CombinatorialAV, Cole2014TheSC, devanur2016sampleauctions, morgenstern2015learning, Syrgkanis2017ASC}, voting and judgment aggregation \cite{Bhattacharyya2015SampleCF, Zhang2018APF}, envy-free allocation \cite{balcan2018efclass}, and optimization \cite{balkanski2017sample,balkanski2016optimization,Rosenfeld2018LearningTO}. 
Some of these works offer low-error approximation guarantees with respect to an optimal solution, such as estimating the maximum \cite{balkanski2017sample,balkanski2016optimization}, maximizing revenue in mechanism design \cite{Cole2014TheSC}, or finding an election winner \cite{Bhattacharyya2015SampleCF}; our work focuses on solutions that minimize {\em expected loss with respect to sampled data}, as is the case when learning the core of a cooperative game \cite{balcan2015learning,igarashi2019learning,sliwinski2017hedonic}, reserve prices in auctions \cite{balcan2018general,morgenstern2016learning}, or approximately efficient allocations \cite{Brero2018CombinatorialAV}. While some of the above works explicitly explore the dimension of the solution space, they do not offer the full generality of our model.

Our analysis of the underlying solution space utilizes recent learning-theoretic tools \cite{daniely2011multiclass}, yielding an extension of classical function dimension measures such as the VC dimension \cite{anthony1999learning,chervonenkis71VC}, and the graph dimension \cite{daniely2011multiclass}.
Our results generalize the General Learning problem discussed in \cite{shalev2010learnability}, as learning solution concepts can also involve some global properties playing a role in the loss function.

%% file: primer.tex
\subsection{Classic PAC Learning}\label{sec:pac-primer}
For the sake of completeness, we provide a brief overview of the PAC learning model.
A {\em learning problem} is defined over an instance space $\cal X$ and a set of functions (the \textit{hypothesis class}) $\HC \subseteq \cal Y^\cal X$ ($\cal Y$ is the {\em label space}). Let $\cal D$ be a distribution over $\cal X\times \cal Y$; we let the {\em loss} of $h \in \HC$ given $\cal D$ be $L_{\cal D}(h) = \Pr_{(x,y) \sim\cal D}[h(x) \ne y]$. 
Given a set of $m$ i.i.d. samples $T= \tup{x_j,y_j}_{j \in [m]}$ from $\cal D$, the {\em empirical loss} of $h \in \HC$ is $\hat L_{T}(h) = \frac1m \sum_{j = 1}^m \Ind(h(x_j) = y_j)$. We assume that $\cal D$ is a distribution over $\cal X$, where every $x \in \cal X$ is evaluated by some unknown $c\in \HC$; this is referred to as the {\em realizable case}, in which there is some $h \in \HC$ for which $L_{\cal D}(h) = 0$, and there is at least one hypothesis $h \in \HC$ for which $\hat L_{T}(h) = 0$ for any $T\subseteq \cal X$. An algorithm $\cal A$ is a PAC learner for $\HC$ if there is some $m_0$ polynomial in $\frac1\eps,\frac1\delta$ and the natural problem parameters, such that for any distribution $\cal D$ and any set of $m \ge m_0$ samples $T$ sampled i.i.d. from $\cal D$, $\cal A$ outputs a hypothesis $h^* \in \HC$ (which is a function of $T$, but not of $\cal D$) such that $\Pr_{T\sim \cal D^m}[L_{\cal D}(h^*) \ge \eps] < \delta$.


For binary hypothesis classes (where the label space is $\cal Y = \{\pm1\}$), the VC dimension \citep{chervonenkis71VC} characterizes the sample complexity of $\HC$. The sample complexity required by any PAC learning algorithm for $\HC$ is upper and lower-bounded by the VC dimension of $\HC$.
This is achieved by an algorithm that outputs a {\em consistent} hypothesis, i.e. one which minimizes the empirical loss $\hat L_T(h)$ w.r.t. a sample $T$. 
\begin{definition}\label{def:vc-dim}
Given a hypothesis class $\HC$, a set $C \subseteq \cal X$ is said to be shattered if for any binary labeling $b:C\to \{0,1\}$ there exists some $h \in \HC$ such that $h(x) = b(x)$ for all $x \in C$.
The VC dimension of $\HC$, or $VC(\HC)$, is the size of the largest set $C \subseteq \cal X$ that is shattered by $\HC$.
\end{definition}
For example, if the hypothesis class is the set of all linear classifiers over $\R^n$, its VC dimension is $\cal O(n)$ \cite{anthony1999learning}. Theorem~\ref{thm:vc-class} relates the VC dimension and the PAC learnability of $\HC$.
\begin{theorem}\label{thm:vc-class}
 There exists absolute constants $\alpha_1$ and $\alpha_2$, such that for a hypothesis class $\HC$, the sample complexity of $\HC$ with respect to $\eps$ and $\delta$ (denoted $m(\eps,\delta)$) is
 $$\frac{\alpha_1}{\eps} \left( \VC(\HC) + \log\left(\frac{1}{\delta}\right) \right) \leq m(\eps, \delta) \leq \frac{\alpha_2}{\eps} \left( \log\left(\frac{1}{\eps}\right)\VC(\HC) + \log\left(\frac{1}{\delta}\right) \right)$$
\end{theorem}

Theorem \ref{thm:vc-class} can be slightly generalized to the following claim: for any two functions $f,g \in \HC$, the empirical loss on an i.i.d. sample of more than $\frac{\alpha_2}{\eps} \left( \log\left(\frac{1}{\eps}\right)\VC(\HC) + \log\left(\frac{1}{\delta}\right) \right)$ points, is close within $\eps$ to the statistical loss (ie. $\underset{x \sim \cal D}{\Pr}[f(x) \neq g(x)]$).


The case where samples are labelled by some arbitrary function $c$ (not necessarily in $\HC$) is also known as the {\em agnostic} case; however, as a result of the uniform convergence results, if an algorithm $\cal A$ outputs a hypothesis $h^*$ that minimizes empirical risk --- $\forall h \in \HC:\hat L_T(h^*) \le \hat L_{T}(h)$ --- the statistical error is $\le \eps$: $\Pr_{T \sim \cal D^m}[L_{\cal D}(h^*) \ge \min_{h \in \HC}L_{\cal D}(h) + \eps] < \delta$. 
Therefore, as discussed in \cite{daniely2011multiclass}, uniform convergence is a powerful tool for bounding statistical loss.


%% file: gametheoryprimer.tex
\subsection{Game-Theoretic Solution Concepts}\label{sec:game-theory-concepts}
In what follows, we briefly introduce the solution concepts discussed in this work. In all scenarios below, we have a set of {\em players} $N = \{1,\dots,n\}$, with preferences over outcomes induced in some manner; our objective is to obtain a {\em solution} with some desirable properties.

\subsubsection{Hedonic Games}\label{sec:hedonicgames-primer}
In hedonic games \cite[Chapter 15]{brandt2016handbook}, each player $i \in N$ has a complete, transitive preference order $\succ_i$ over coalitions in $N$ that contain it. Solutions are partitions (also referred to as {\em coalition structures}) of $N$; a coalition structure $\pi$ is {\em blocked} by a coalition $S \subseteq N$ if all members of $S$ prefer $S$ over the coalition they are in (denoted $\pi(i)$), i.e. $S\succ_i \pi(i)$ for all $i \in S$. The {\em core} of a hedonic game is the set of {\em stable} coalition structures: they cannot be blocked by any coalition $S \subseteq N$. 
It is often assumed that players' preferences over subsets are induced by a {\em cardinal} utility function $v_i:2^N \to \R_+$; in this case, $S \succ_i T$ if and only if $v_i(S) > v_i(T)$. 

\subsubsection{Competitive Equilibria in Fisher Markets}\label{sec:markets-primer}
We are given a set of $k$ indivisible goods $G = \{g_1,\dots,g_k\}$. Each player $i \in N$ values bundles of goods in $G$ according to $v_i:2^G \to \R_+$, where $v_i(\emptyset) = 0$ for all $i \in N$. A {\em market outcome} is a tuple $\tup{\pi,\vec p}$, where $\pi$ is a partition of $G$ into $n$ disjoint bundles (some of them may be empty), with $\pi(i)$ assigned to player $i$; $\vec p \in \R^k$ is a price vector, denoting the price of each item in $G$.
In these markets, known as {\em Fisher markets}~\cite{budish2011assignment}, we assume that each player $i$ has a {\em budget} $\beta_i \in \R_+$. 
Given a price vector $\vec p \in \R^k$, the {\em affordable set} of player $i$ is the set of all bundles whose total price is less than $\beta_i$:
$$\cal A_i(\vec p,\beta_i) = \left\{S \subseteq G: \sum_{g_j \in S}p_j \le \beta_i \right\}.$$
An outcome $\tup{\pi,\vec p}$ is a {\em competitive equilibrium} if for all $i \in N$, $\pi(i) \in \cal A_i(\vec p,\beta_i)$, and $\forall S \in \cal A_i(\vec p,\beta_i)$, $v_i(\pi(i)) \ge v_i(S)$.

\subsubsection{Condorcet Winners}\label{sec:voting-primer}
Consider a set of {\em voters} $N = \{1,\dots,n\}$, each with a {\em preference order} $\succ_i$ over some finite set of {\em candidates} $C$. Given two candidates $c,c' \in C$, we define $B(\succ,c',c) = 1$ iff a majority of voters prefer $c'$ to $c$ under $\succ$. A candidate $c^*$ is a {\em Condorcet winner} iff $B(\succ,c^*,c)=1$  for every other candidate $c \in C$. 

%% file: sol-concepts.tex
\section{A PAC Framework for Distribution-Based Solution Concepts}\label{sec:pac-sol-concepts}
As described in Section~\ref{sec:game-theory-concepts}, a \textit{solution concept} or an \textit{equilibrium concept} characterizes a subset of its solution space satisfying some natural desiderata. 
{\em Games} are mappings from some domain $\cal X$ to a {\em label space} $\cal Y$. For example, in hedonic games, a game is a set of functions $v_i:2^N \to \R$ (for every $i \in N$), mapping from subsets of players to real values; thus, $\cal X = 2^N$ and $\cal Y$ consists of vectors of the form $(v_i(S))_{i \in S}$ for every $S\subseteq N$. 
We assume no knowledge of the actual game $g$, except for the hypothesis class it belongs to; we only observe samples of the game's evaluation on points in $\cal X$. 

Constraints characterizing solution concepts are often universal quantifiers over a {\em local loss function} $\lambda$. In hedonic games, we define $\lambda: 2^N \times \G \times \Pi(N) \rightarrow \{0,1\}$, where $\lambda(S,\vec v,\pi)=0$ iff $v_i(S) \le v_i(\pi(i))$ for all $i \in N$. Thus, $\pi$ is in the core iff $\lambda(S,\vec v,\pi)=0$ for all $S \subseteq N$.
Similarly, the maxima of a function satisfy $x^{*} \in \argmax f(x) \iff \forall x : f(x) \leq f(x^{*})$; in particular, $\lambda(x,f,x^*) = 0$ iff $f(x^* ) \ge f(x)$.
Solution concepts that can be defined via a local loss function $\lambda$ readily admit a distributional variant: we require that the expected loss as measured by $\lambda$ is low, with respect to a distribution $\cal D$ over the domain $\cal X$; i.e. $\underset{x \sim \cal D}{\Pr}[\lambda(x,g,s) = 0] \ge 1 - \eps$, for some $\eps \in (0,1)$. 

More formally, an instance of the \probleminstancename problem is a tuple $\probleminstance=(\cal X,\cal Y,\G,\Sol,\lambda)$. Here $\cal X$ is the instance space; $\cal Y$ is the codomain (or label) space; $\G \subseteq \cal Y^{\cal X}$ is the class of games; $\Sol$ is the solution space; finally, $\lambda : \cal X \times \G \times \Sol \rightarrow \{0,1 \}$ measures \textit{local loss}.
In standard PAC learning (Section~\ref{sec:pac-primer}), $\G = \Sol$ and $\lambda(x,g_0,g_1) = 0 \iff g_0(x) = g_1(x)$. 

Given a game $g \in \G$ and $m$ points $T = \langle (x_j,g(x_j)) \rangle_{j = 1}^m$, the {\em empirical error} (or empirical risk) of $s\in \Sol$ is
$$\hat L_T(g,s) = \frac{1}{m} \sum\limits_{(x_j,g(x_j)) \in T} \lambda(x_j,g,s),$$
and the statistical error (or statistical risk) as 
$L_{\cal D} (g,s) = \underset{T \sim {\cal D}^m}{\E}[ \lambda(x_j,g,s)]$.
A PAC solver for a \probleminstancename $\probleminstance$ is an algorithm $\cal L$ whose input is a list $T = \langle x_j,g(x_j)\rangle_{j = 1}^m$ of $m$ values $x_j \in \cal X$ labelled by some unknown $g\in \G$, and whose output is a solution $s^* \in \Sol$; its sample complexity, denoted $m_{\cal L}(\eps,\delta)$, is the minimal number of samples required such that for any $m \ge m_{\cal L}(\eps,\delta)$, $\underset{T \sim {\cal D}^m}{\Pr}[L_{\cal D} (g,s^*) > \eps] < \delta$. 
We let $m_\probleminstance^\PAC(\eps,\delta)$ be the minimal sample complexity $m_{\cal L}(\eps,\delta)$ required by any PAC solver $\cal L$ for $\probleminstance$.

\subsection{Consistent Solvers and Barriers of Indistinguishability}
In standard PAC learning, {\em consistent} or {\em empirical risk minimizing (ERM)} solvers play an important role; these are algorithms that minimize empirical error ($\hat L_{T}(h)$) on the training sample. As discussed in Section \ref{sec:pac-primer}, for binary functions, consistent algorithms are PAC learners whose sample complexity is bounded by the VC dimension. 
We first define a notion of consistency for solution concepts.
\begin{definition}\label{def:consistency}
An algorithm $\cal A_m : (\cal X \times \cal Y)^m \rightarrow \Sol$ is said to be a {\em consistent solver} for the \probleminstancename problem $\probleminstance$ if for all $g \in \G$, and for any set of $m$ samples $T_m = \langle x_j,g(x_j)\rangle_{j = 1}^m$, $\cal A_m$ takes as input $T_m$ and outputs a solution $s^* = \cal A_m(T_m) \in \Sol$ such that the empirical loss of the solution $s^*$ over $T_m$ is $0$: $\hat L_{T_m}(g,s^*) =0$.
In other words, for any input batch $T_m$ labelled by some underlying function $g \in \G$, the algorithm returns a solution that has zero loss w.r.t $g$ on all points in the input sample.
\end{definition}

\begin{yaircomment}
		I completely removed the randomized consistency theorem+proof. We don't need it.
\end{yaircomment}

The definition of consistent solving presents a subtle yet crucial departure from the corresponding result in standard PAC learning. In PAC learning, since $\lambda(x,g,h) = \Ind[g(x)=h(x)]$, if there are two functions $g_0, g_1 \in \G$ such that $g_0(x) = g_1(x)$ for a point $x \in \cal X$, then for any hypothesis $h$, $\lambda(x,g_0,h) = \lambda(x,g_1,h)$. 
Therefore, even if two functions $g_0, g_1 \in \G$ generate an equivalent sample $T = \langle x_j,y_j \rangle_{j =1}^m = \langle x_j,g_0(x_j) \rangle_{j =1}^m = \langle x_j,g_1(x_j) \rangle_{j=1}^m$, if $h$ is consistent with samples $(x_j,y_j)$ in $T$, then it is consistent with both $g_0$ and $g_1$. 
In fact, this implies that, time complexity considerations aside, a consistent solution always exists in standard learning, and can be found via exhaustive search. 
This is not the case in solution concept learning; two functions $g_0, g_1 \in \G$ may generate an equivalent sample $T = \langle x_j,y_j \rangle_{j = 1}^m $, yet disagree on a solution (this is noted in prior works \cite{igarashi2019learning,sliwinski2017hedonic}). 
Intuitively, this occurs since game-theoretic solutions treat unobserved regions of players' preferences. For example, in PAC market equilibria, one must inevitably set prices for unobserved goods, and assign bundles to players without knowing what their value might be; in hedonic games, a partition of players may contain subsets completely unobserved in the sample data.
It is often useful to think of domains where this issue does not occur, as captured in the following definition. 
Given a labelled sample of $m$ points $T\in (\cal X\times \cal Y)^m$, let $\G|_{T}$ be the set of games in $\G$ which agree with $T$.
\begin{definition}\label{def:consistent-solvability}
A \probleminstancename $\probleminstance$ is said to satisfy the {\em consistent solvability criterion} if for all $m$ and all $T \subseteq (\cal X\times \cal Y)^m$, there exists some $s \in \Sol$ such that 
for all $g \in \G|_{T}$, the empirical loss $\hat L_T(g,s)$ is $0$.
\end{definition}
\begin{yaircomment}
	Please note the changes made to Definition~\ref{def:consistent-solvability}! 
\end{yaircomment}

%% file: vctheory2.tex
\subsection{A Dimension Theory for Game-Theoretic Solutions}\label{sec:vctheory}
We now present a novel definition of dimension for the PAC solution setting, and use it to bound the sample complexity for finding solutions to problem domains. 
\begin{restatable}[Solution-based Dimension]{definition}{defSshattering}\label{def:S-shattering}
Given some $C \subseteq \cal X$, we say the set $C$ is S-shattered in $\probleminstance$ if there exists a game $g \in \G$, such that for every binary labelling $b:C \to \{0,1\}$ there exists a solution $s \in \Sol$ (that may depend on $b$) such that for all $x \in C$, $\lambda(x,g,s) = b(x)$.

The Solution-based dimension of $\probleminstance$, denoted $\Sd(\probleminstance)$, is the size of the largest set S-shattered in $\probleminstance$, and $(C, g)$ as the corresponding shattering witness.
\end{restatable}

$\Sd(\probleminstance)$ bounds the sample complexity of consistent solutions for $\probleminstance$ (i.e. $m^\PAC(\eps,\delta)$); however, we first prove a stronger claim, using the idea of uniform convergence discussed in Section~\ref{sec:pac-primer}. If we define the sample complexity for uniform convergence $m^{\UC}(\eps, \delta)$ as the number of samples required such that the empirical loss of any solution is $\eps$-close to its statistical loss, then $m^{\UC}(\eps,\delta)$ is polynomially dependent on the solution dimension of the problem.

\begin{restatable}{theorem}{thmupperbound}\label{thm:upperbound}
There are universal constants $\alpha_1$ and $\alpha_2$, such that if $\Sd(\probleminstance) = d$, then for a sample of $m \geq \alpha_1 \frac{d + \log(\frac1\delta)}{\eps^2}$ points $T = \langle x_j, y_j\rangle_{j=1}^m$,
$$\underset{T \sim {\cal D}^m}{\Pr}[\exists g \in \G|_{T}, s \in \Sol: |\hat L_T (g, s) - L_{\cal D} (g, s)| > \eps] < \delta.$$
Furthermore, if a solution $s^*$ is consistent, i.e. $\hat L_T(g,s^*) = 0$, then for any $m$ greater than $\frac{\alpha_2}{\eps} \left( \log\left(\frac{1}{\eps}\right) d + \log\left(\frac{1}{\delta}\right) \right)$, we have that $\underset{T \sim {\cal D}^m}{\Pr}[L_{\cal D} (g, s^*) > \eps] < \delta$.
\end{restatable}
Note that in particular, $m^{\PAC}(\eps, \delta) \leq m^{\UC}(\eps,\delta)$, and both are polynomially dependent on $\Sd(\probleminstance),\frac1\eps$ and $\log\frac1\delta$.

As a useful sanity check, we observe that $\Sd$ collapses to the classic VC dimension when learning classifiers: when $\Sol = \G = \HC \subseteq 2^{\cal X}$, then  $\Sd(\probleminstance) = \VC(\HC)$. Similarly, when $\Sol = \G = \HC \subseteq {\cal Y}^{\cal X}$ (i.e. for multiclass learning problems with a general domain $\cal Y$), $\Sd$ collapses to the graph dimension \cite{daniely2011multiclass}.
We now observe few immediate corollaries of the above uniform convergence result.

\begin{restatable}{corollary}{corrconjuncts}\label{corr:conjuncts}
Given a \probleminstancename problem $\probleminstance = \tup{\cal X,\cal Y,\G,\Sol}$:

\noindent \textbf{Simultaneous Constraints:} if multiple local loss functions $\lambda_1, \dots, \lambda_k$ need to be simultaneously approximated within $\eps$, i.e. $\forall i \in [k]: \lvert {\hat L_i}(g,s) - {L_i}_{\cal D}(g,s) \rvert < \eps$, then the sample complexity of finding a solution satisfying all of them is in $\cal O(\underset{i \in [k]}{\max}\{ m^{\UC}_i(\eps, \delta)\})$.

\noindent \textbf{Separable Conjunctions:} if there are local constraints $\lambda_1$ over $\Sol_1$, and $\lambda_2$ over $\Sol_2$, where $\Sol_\probleminstance = \Sol_1 \times \Sol_2 $, such that we need to bound their conjunction within $\eps$, i.e. $\Pr[\lambda_1(x,g,s_1) \wedge \lambda_2(x,g,s_2)]$, then $ m^{\UC}(\eps, \delta)$ is in $\cal O(\underset{i \in \lbrace 1,2 \rbrace}{\max} \{\Sd(\probleminstance_i)\})$.

\end{restatable}

The proof of Corollary \ref{corr:conjuncts} is relegated to the appendix.
The following claim (whose proof is also relegated to the appendix) is also useful 
\begin{restatable}[$\Sd$ for Argmax]{corollary}{corrargmax}\label{corr:argmax}
Let $\probleminstance_{\max}$ be defined by $\G = \{f:\cal X\to \cal Y\}$ and $\Sol = \cal X$, where $\cal Y$ is endowed with a total order $\succ$, and $\lambda(x, g, x^{*}) = \Ind[g(x) \succ g(x^{*})]$. Then, $\Sd(\probleminstance_{\max})=1$.
\end{restatable}
To conclude, in order to establish an efficient PAC algorithm for a problem $\probleminstance$, it suffices to upper-bound $m(\probleminstance)$ by its solution dimension $\Sd(\probleminstance)$.

\subsection{Uniform Convergence Beyond Consistency}\label{sec:agnostic}

As discussed in Section \ref{sec:pac-primer}, uniform convergence complexity bounds can bound the sample complexity for Agnostic PAC learning via Empirical Risk Minimizers (ERM learners). 
However, agnostic solution learning can be defined in many ways. 
We discuss two definitions, for which the corresponding notion of {\em ERM solving} has a sample complexity that follows from uniform convergence.
\begin{definition}
For a given \probleminstancename $\probleminstance$, $\cal A$ is a 
worst-case Agnostic PAC Solver if
$$\underset{T \sim {\cal D}^m}{\Pr}[L_{\cal D}(g,\cal A(T)) \leq \underset{s}{\min} \underset{g^\prime \in \G|_{T}}{\max} L_{\cal D}(g^\prime, s) + \eps] \geq 1 - \delta;$$
it is a Bayesian agnostic PAC Solver, for a prior over games $\tilde{\cal D}$, if 
$$\underset{T \sim {\cal D}^m}{\Pr}[\underset{g^\prime \sim \tilde{\cal D}}{\E}[L_{\cal D}(g^\prime,\cal A(T))| g^\prime \in \G|_{T}] \leq \underset{s}{\min} \underset{g^\prime \sim \tilde{\cal D}}{\E} [L_{\cal D}(g^\prime, s)| g^\prime \in \G|_{T}] + \eps] \geq 1 - \delta.$$
\end{definition} 
\begin{restatable}{corollary}{corrpacsolvingERMbounds}\label{corr:pac-solving-ERM-bounds}
Given some $\probleminstancename$, the sample complexity for Worst-Case and Bayesian agnostic PAC solving is in $\cal O(\Sd(\probleminstance))$, and is achievable by an empirical risk minimizer.
\end{restatable}
The proof of Corollary~\ref{corr:pac-solving-ERM-bounds} is relegated to the appendix.

%% file: concretizations.tex
\section{Learning Game-Theoretic Distribution-Based Solution Concepts}\label{sec:comp-solutions}
Let us now apply our theory for learning solution concepts in game-theoretic domains; all problems described below follow a common theme: rather than learning preferences, we learn solutions using the sampled dataset.
While the focus of this paper is on game-theoretic solutions, our theory applies for other types of solution concepts as long as one can define a local loss function $\lambda$ that depends only on a given point in $x\in \cal X$, $g \in \G$ and the solution $s \in \Sol$ (see Section \ref{sec:pac-sol-concepts}).

\input{coopgames}

\input{markets}

\subsection{PAC Condorcet Winners}\label{sec:condorcet}
We conclude with a discussion of statistical solution concepts in voting (see Section~\ref{sec:voting-primer} above). 
A PAC Condorcet winner is a candidate $c^*$ such that $\Pr_{c \sim \cal D}[B(\succ,c,c^*)] < \eps$ (recall that $B(\succ,c,c^*) = 1$ iff a majority of voters prefer $c$ to $c^*$). 
We refer to the problem of finding a Condorcet winner as $\probleminstance_{\Cond}$.
We require that given a sample $T \subseteq C$ of candidates, we can infer voters' preferences w.r.t. $T$. This can be encoded as a valuation function of $i$ over the candidates (as is the case for hedonic games, see Section~\ref{sec:hg-core}), or the truncated ranking $\succ_i$ over the sampled candidates for every $i \in N$. 
Given a class of preference profiles $\HC$, let $\probleminstance_{\Cond}(\HC)$ be the problem of finding Condorcet winners for profiles in $\HC$.
We define the {\em tournament graph}: this is a directed graph where candidates are nodes; given a preference profile $\succ$, there is an edge from $a$ to $b$ if $a$ beats $b$ in a pairwise election under $\succ$.
\begin{theorem}\label{thm:sample-condorcet}
Given a class of preference profiles $\HC$ over $C$ such that $|C|>1$, and a sample of candidates $T \subseteq C$, the following are equivalent:
\begin{inparaenum}[(a)]
\item There exists a consistent solver for $\probleminstance_{\Cond}(\HC)$ that returns a PAC Condorcet winner $c^* \in T$.
\item $\Sd(\probleminstance_{\Cond}) = 1$.
\item for every preference profile $\succ  \in \HC$, the tournament graph is transitive.
\end{inparaenum}

In particular, if $\HC$ satisfies the above, there exists a PAC solver for $\probleminstance_{\Cond}(\HC)$ whose sample complexity is $\frac{1}{\epsilon}log\frac{1}{\delta}$.
\end{theorem}
\begin{proof}
If there is a preference profile $h \in \HC$ for which the tournament graph contains a $3$-cycle, then there immediately exist two vertices of that cycle that can be S-shattered. This is true since for every $\HC$ with more than one candidate, every singleton is shattered. Therefore, $\Sd(\probleminstance_{\Cond}) = 1 $ if and only if there are no preference profiles with Condorcet $3$-cycles, which is equivalent to transitivity.
Similarly, the existence of Condorcet winner for every $C^\prime \in C$ is equivalent to absence of any cycles, which is equivalent to transitivity of the tournament graph.
\end{proof}
Two notable families of voter preferences exhibit transitive preferences: single peaked preferences \cite[Chapter 2]{brandt2016handbook} and single-crossing preferences \cite{gans1996singlecrossing} (see \cite{elkind2017structured} for an overview); thus, if $\HC$ is any of the former, a Condorcet winner can be PAC learned using $\frac1\eps\log\frac1\delta$ samples.
Whenever the Condorcet winner is known to exist within a sample $C' \in C$, the problem is equivalent to the argmax problem discussed in Corollary \ref{corr:argmax}. 
However, as shown in Section \ref{sec:agnostic}, the graph dimension is still useful as a means to estimate  (within $\pm \eps$ with high confidence) the behavior of a candidate in pairwise elections using a small empirical sample, even when no Condorcet winner exists. Theorem~\ref{thm:no-condorcet} bounds $\Sd(\probleminstance_{\Cond})$ in the case where Condorcet winners do not exist. 
The result bounds the solution dimension in terms of the underlying structure of the tournament graph, and is based on Corollary~\ref{corr:pac-solving-ERM-bounds}; the full proof is in the appendix.
\begin{restatable}{theorem}{thmnocondorcet}\label{thm:no-condorcet}
	Let $k$ be the largest number of candidates, such that for some tournament graph in $\HC$, every pair among them is part of some $3$-cycle. Then $\Sd(\probleminstance_{\Cond}(\HC)) \leq log_2 (k+2)$.
\end{restatable}

%% file: coopgames.tex
\subsection{The PAC Core for Hedonic Games}\label{sec:hg-core}
Let us begin with hedonic games (Section~\ref{sec:hedonicgames-primer}); we analyze another type of cooperative game (TU cooperative games) in the appendix.
A partition $\pi^*$ of $N$ PAC stabilizes a hedonic game w.r.t. a distribution $\cal D \in \Delta(2^N)$ (where $i \in \pi^*(i)$ for every $i \in N$), if 
$\Pr_{S\sim \cal D}\left[\forall i \in S: v_i(S) \geq v_i(\pi^*(i)) \right] < \eps$.
The local loss function $\lambda$ takes as input a coalition $S \subseteq N$, players' valuations $\vec v = \tup{v_1\,\dots,v_n}$ and a partition $\pi \in \Pi(N)$; $\lambda(S,\vec v,\pi) = 1$ iff $S$ can block $\pi$ under $\vec v$.
Our key result here is that the sample complexity of PAC stabilzing hedonic games is linear in $n$, {\em for any class $\cal H$ of games}.

\begin{lemma}\label{lem:hedonic-gdim}
	For any class of Hedonic Games $\cal H$ over $n$ players, the solution dimension of PAC stabilizing $\cal H$ is $\le n$.
\end{lemma}
\begin{proof}
	By definition, for a given hedonic game $h \in \cal H$, a partition $\pi$, and a coalition $S \subseteq N$, the local loss $\lambda(S,h,\pi)=0$ if and only if there exists a player in $S$ that does not prefer it over her assigned coalition in $\pi$, i.e. $v_i(S) < v_i(\pi(i))$.
	If a set of $m$ coalitions $\cal S = \{S_1,\dots,S_m\}$ is S-shattered by a witness $h \in \cal H$, then for each $S_j \in \cal S$, there exists a coalition structure $\pi_j$ such that $\lambda(S_j, h, \pi_j) = 0$, but $\lambda(S_k,h,\pi_j) = 1$ for all $k \ne j$. In other words, under $\pi_j$, there exists some $i \in S_j$ such that $v_i(S_j) < v_i(\pi_j(i))$, and for all $k \neq j$ and for all $i \in S_k$, $v_i(S_k) \geq v_i(\pi_j(i))$. 
	We conclude that for every $S_j \in \cal S$, there exists a player $i$ who strictly prefers all coalitions that she belongs to in $\cal S \setminus \{S_j\}$ over $S_j$. More formally, we let $\cal T(i)$ be the set of coalitions which are least preferred by player $i$ in $\cal S$; note that $\cal T(i)$ must be a singleton, or else we arrive at a contradiction (the least liked coalition must be unique). Therefore, if $\cal S$ is S-shattered, the number of coalitions in $\cal S$ is bounded by $n$, and we are done. 
\end{proof}

Applying Theorem \ref{thm:upperbound} and leveraging Lemma~\ref{lem:hedonic-gdim} we obtain the following result:
\begin{theorem}\label{thm:hedonic}
	A class of Hedonic games $\cal H$ is efficiently PAC stabilizable iff there exists an algorithm that outputs a partition consistent with samples evaluated by a game $g \in \cal H$; the sample complexity in this case is $\cal O(n)$. 
\end{theorem}
In particular, \citet{sliwinski2017hedonic} propose a consistent algorithm for {\em top-responsive} hedonic games \cite[Chapter 15]{brandt2016handbook}; \citet{igarashi2019learning} present a consistent algorithm for hedonic games whose underlying interaction graph is a tree~\cite{igarashi2016hedonicgraph}. Indeed, given Theorem~\ref{thm:hedonic}, it suffices to show that the algorithms they propose are consistent; their correctness is immediately implied by our results. 


%% file: markets.tex
\subsection{PAC Competitive Equilibria}\label{sec:comp-eq}
Competitive equilibria (CE) readily admit a PAC variant:
given an allocation $\pi$, let $\cal P_i(\pi)= \{S \subseteq G: v_i(S) > v_i(\pi(i))\}$ be the set of bundles that are strongly preferred by $i$ to $\pi(i)$. 
One can think of a CE as an outcome that ensures that $\cal P_i(\pi)\cap \cal A_i(\vec p,\beta_i)= \emptyset$, i.e. $i$ cannot afford any bundle that it prefers to its assigned bundle. In the statistical variant, we wish to ensure that this intersection has a low measure under a distribution $\cal D$ over $2^G$.
We define a loss $\lambda_i$ per player $i$ as follows: given a bundle of goods $S \subseteq G$, player valuations $\vec v$ and a market outcome $\tup{\pi,\vec p}$, $\lambda_i(S,\vec v,\tup{\pi,\vec p}) = 1$ iff $S$ is both affordable (in $\cal A_i(\vec p,\beta_i)$), and is preferred to $\pi(i)$ (in $\cal P_i(\pi)$). Our objective is to ensure that the overall error of player loss functions $\lambda_1,\dots,\lambda_n$ are within an error of $\eps$.
Lemma~\ref{lemma:pac-ce-gdim} bounds the sample complexity for PAC learning this problem, $m^{\PAC}(\eps, \delta)$, by $O(k)$. Therefore, by Theorem \ref{thm:upperbound} and Corollary \ref{corr:conjuncts}, any algorithm that generates an outcome consistent against $m$ sampled bundles would also be a PAC CE solver with a sample complexity in $O(k)$. We refer to an instance of the CE problem as $\probleminstance_{\CE}(N,G,\vec v,\vec \beta)$.
\begin{lemma}\label{lemma:pac-ce-gdim}
The solution dimension $\Sd(\probleminstance_{\CE}(N,G,\vec v,\vec b))$ is $O(k)$, where $k = |G|$. 
\end{lemma}
\begin{proof}
For every player $i\in N$, the local constraint $\lambda_i$ can be seen as a conjunction of $\lambda_{1,i}(S,
\vec v,\tup{\pi^*,\vec p^*})  = \Ind\left[\sum_{g_j \in S}p_j^* > \beta_i\right]$,
and $\lambda_{2,i}(S, \vec v, \tup{\pi^*,\vec p^*}) = \Ind[v_i(S) > v_i(\pi^*(i))]$.
Since $\lambda_{1,i}$ is defined by a linear constraint set by $\vec p^*$ and $\beta_i$, it can be S-shattered by $O(k)$ samples (in a manner similar to linear separators in standard PAC learning), which bounds its S-dimension. On the other hand, every $\lambda_{2,i}$ is a simple argmax constraint, which by Corolllary \ref{corr:argmax}, has a solution dimension of 1. By applying Corollary \ref{corr:conjuncts}, the dimension of $\lambda_i = \lambda_{1,i} \wedge \lambda_{2,i}$ is $O(k)$; since the condition of $\lambda_i$ must hold for each $i \in N$, the CE loss is given by $\lambda = \bigwedge_i \lambda_i$, which is $O(k)$ by Corollary \ref{corr:conjuncts}.
\end{proof}
Lemma~\ref{lemma:pac-ce-gdim} bounds the dimension of $\probleminstance_{\CE}$ by $O(k)$; however, the challenge is to design algorithms that generate consistent market solutions: bundle assignments and prices that ensure that all observed goods have been allocated, with no excess demand or assignment. 
We show the existence of consistent solutions in two different settings; however, our solutions relax the market constraints. For Fisher markets with budgets $\vec \beta$, for any $\zeta > 0$, there exists a perturbed budget vector $\vec \beta^*$ with $\lVert\vec \beta^* - \vec \beta \rVert_\infty \leq \zeta$ for which there exists a consistent solution $\tup{\pi^*,\vec p^*}$ w.r.t. $\vec \beta^*$; this result holds for {\em any class of valuation functions}. Theorem~\ref{thm:fisher-indivisible} utilizes {\em inefficient} market outcomes, where a good may be allocated to more than one person; it is easy to think of an allocation $\pi$ as a list of vectors in $\{0,1\}^k$, where $\pi_j(i) = 1$ iff the $j$-th good is allocated to player $i$. If all goods are allocated, then $\sum_{i \in N} \pi(i) = \vec 1$; if goods are over-allocated, then $\sum_{i \in N} \pi(i) > \vec 1$.
\begin{theorem}\label{thm:fisher-indivisible}
	We are given $\probleminstance_{\CE}(N,G,\vec v,\vec \beta)$, and $m$ sampled bundles $S_1, \dots, S_m\subseteq G$ evaluated by $\vec v$. For any $\zeta > 0$, there exists a perturbation on $\vec \beta$, $\vec \beta^*$ such that $\lVert \vec \beta - \vec \beta^*\rVert_\infty< \zeta$, for which there is an outcome $\tup{\pi^*,\vec p^*}$ such that players with budget levels $\vec \beta^*$ do not demand $S_1,\dots,S_m$; moreover, $\lVert\sum_{i \in N} \pi^*(i) - \vec{1} \rVert_2 \leq \frac{k}{2} $, where $\vec{1} = (1,1,\dots,1) \in [0,1]^k$.
\end{theorem}
\begin{proof}
We restrict ourselves to finding an assignment using only the sampled bundles and the empty bundle, i.e. for all $i\in N$: $\pi^*(i) \in \{\emptyset,S_1,\dots,S_m\}$; thus, we avoid making any assumptions about the structure of $v_i$.
\citet[Theorem 1]{budish2011assignment} shows that given $\vec \beta$ such that $\max_i \beta_i>\min_i \beta_i$, for any $\zeta > 0$ there exists a perturbed budget vector $\vec \beta^*$ and an outcome $\tup{\pi^*,\vec p^*}$ for which:
$\pi^*(i) \in \argmax_{S \in \cal A_i(\vec p^*,\beta_i)} v_i(S)$; $\lVert\vec \beta - \vec \beta^* \rVert_\infty < \zeta$ and $\lVert\sum_{i \in N} \pi^*(i) - \vec{1} \rVert_2 \leq \frac{k}{2} $.
Assuming that for every other $S \notin \{\emptyset,S_1,\dots,S_m\}$, $v_i(S) \leq 0$ for all $i \in N$, and applying the result by \citeauthor{budish2011assignment}, there exists a consistent outcome satisfying our requirements. 
\end{proof}
While it makes no assumptions on player valuations, Theorem~\ref{thm:fisher-indivisible} is not constructive: it relies on a classic result from \citet{budish2011assignment}, which utilizes a fixed-point theorem by \citet{And91fixedpoint} for discontinuous maps to bound excess demand. 
We analyze {\em exhange economies}, a market variant with divisible goods, in the appendix. In both cases, we are able to show that consistent market solutions exist. However, our results show the existence of solutions which only {\em partially satisfy} the equilibrium guarantees; moreover, both cases utilize non-constructive fixed-point theorems, rather than provide an efficient algorithm. There is little reason to believe that consistent solutions can be easily computed in the general case; finding market solutions in settings similar to ours is PPAD complete \cite{Othman14thecomplexity}.


%% file: conclusion.tex
\section{Conclusions and Future Work}\label{sec:conclusion}
We propose a formal, general framework for learning solution concepts from data, and apply it to several problems in economic domains. While several solution concepts have a polynomial sample complexity, efficiently computing a consistent solution remains a challenging open problem. 
In the case of market equilibria, we believe that there exist consistent algorithms for specific valuation classes, such as gross-substitutes~\cite{gul1999walrasian} or submodular valuations. 
While we mostly focus on the realizable case, solving the non-realizable case is an interesting open problem. 
Our model easily accommodates {\em approximate solutions} (as we do for market equilibria) by assimilating the approximation guarantee into the loss function; this can be done generally by adopting the PMAC learning framework \cite{balcan2011submodular}. 
Our work upper-bounds the solution dimension using a generalization of the graph dimension; however, we offer no lower bounds. \citet{daniely2011multiclass} use the {\em Natarajan dimension}~\cite{natarajan1989learning} to establish lower bounds in multiclass learning; using the Natarajan dimension to lower-bound the solution dimension is a promising direction for future work. 


%% file: apdxvctheory.tex
\section{Missing Proofs for Section 4.2}
\label{apdx:vctheory}

We now present the proof for the Theorem \ref{thm:upperbound} that provides an upper bound to the sample complexity for uniform convergence in terms of the \textit{solution dimension}.

\defSshattering*

\thmupperbound*

\begin{proof}
For any $g \in \G$, consider $\Phi_g = \lbrace \phi_{g,s}: \cal X \rightarrow \lbrace 0,1 \rbrace | s \in \Sol, \phi_{g,s} (x) = 1 - \lambda(x, g, s) \rbrace$. Observe that, $\Sd(\probleminstance) = \underset{g}{\max} \, \VC(\Phi_g)$; this implies that if $\Sd(\probleminstance) = d$, then for every $g \in \G$, $|\Phi_g| \leq |\cal X|^d$.

From Theorem \ref{thm:vc-class}, we know that for any $g \in \G|_{T}$ and $f \in 2^{\cal X}$, if $T$ contains at least $m\geq \alpha \frac{d + \log(\frac1\delta)}{\eps^2}$ samples (for the appropriate constant $\alpha$, then (where $L_T^\phi$ and $L_{\cal D}^\phi$ denote the corresponding loss functions for binary functions):
\begin{equation}\label{eqn:gdimproof1}
\underset{T \sim {\cal D}^m}{\Pr}[\exists \phi_{g,s} \in \Phi_g: |L_T^\phi (\phi_{g,s}, f) - L_{\cal D}^\phi (\phi_{g,s}, f)| > \eps] < \delta
\end{equation}

Let $\Ind: \cal X \rightarrow \lbrace 0,1 \rbrace$, such that $\forall x: \Ind(x) = 1$. Then, it follows that $\lambda(x,g,s) = \I[\phi_{g,s}(x) \neq \Ind(x)]$. By substituting the loss functions in Eq~\eqref{eqn:gdimproof1}, and taking $f = \Ind$, we get:
$$\underset{T \sim {\cal D}^m}{\Pr}[\exists g \in \G|_{T}, s \in \Sol: |\hat L_T (g,s) - L_{\cal D} (g,s)| > \eps] < \delta.$$

This proves that $m^{\UC}(\eps,\delta)$, and $m^{\PAC}(\eps, \delta)$ are both polynomially bounded by $\Sd(\probleminstance)$. 
If a solution is consistent, then we also know that there exists some $s^{*} \in \Sol$ such that for all $g \in \G|_{T}$, $\phi_{g, s^{*}} = \Ind \in \Phi_g$. The second part of the statement similarly follows from the upper bound of $m^{\PAC}$ for binary functions in Theorem \ref{thm:vc-class}.
\end{proof}
Let us next prove Corollary~\ref{corr:conjuncts}. 
\corrconjuncts*
\begin{proof}
	Part 1 (Simultaneous Constraints) is a direct corollary of Theorem \ref{thm:upperbound}; if $m \geq \underset{i}{\max} \{m^{\PAC}_i(\eps, \delta)\}$, then $$\forall i: \underset{T \sim {\cal D}^m}{\Pr}[ \underset{x \sim \cal D}{\Pr}[\lambda_i(x,g,s)] < \eps ] < \delta.$$
	
	For part 2 (Separable Conjunctions), we are given a loss function $\lambda: \cal X \times \G \times (\Sol_1 \times \Sol_2) \rightarrow  \{0,1 \}$ with $\lambda (x, g, (s_1, s_2)) = \lambda_1 (x,g,s_1) \wedge \lambda_2 (x,g,s_2)$, such that for any $g \in \G$, we wish to find $ (s_1, s_2) \in \Sol_1 \times \Sol_2$ that bound the gap between the empirical and statistical loss. 
	
	For any $s_2 \in \Sol_2$, let us define $\lambda_{1|s_2}: \cal X \times \G \times \Sol_1 \rightarrow  \{0,1 \}$ with $\lambda_{1|s_2} (x, g, s_1) = \lambda (x, g, (s_1, s_2)) = \lambda_1 (x,g,s_1) \wedge \lambda_2 (x,g,s_2)$. Observe that $\lambda_{1|s_2} (x, g, s_1)$ equals $\lambda_1 (x,g,s_1)$ whenever $\lambda_2 (x,g,s_2) = 1$, and is  otherwise $0$ for every $s_1$.
	We note that a set shattered under $\lambda_{1|s_2}$  cannot contain any point $x^\prime$ such that $\lambda_2 (x^\prime,g,s_2) = 0$ (points for which $\lambda_2(x^\prime,g,s_2) = 0$ always evaluate to $0$ under $\lambda$ and do not admit Boolean functions $b$ for which $\lambda(x^\prime,g,(s_1,s_2)) = b(x') = 1$); thus, we know that a set shattered in $\lambda_{1|s_2}$ is also shattered under $\lambda_1$, therefore the solution dimension corresponding to $\lambda_{1|s_2}$ is bounded by the solution dimension corresponding to $\lambda_1$. 
	We also observe that the formulas for empirical and statistical loss under $\lambda_{1|s_2}$ and $\lambda$ are equivalent; therefore, by Theorem \ref{thm:upperbound}, we know that if $m \geq m_1^{\UC}(\eps, \delta)$
	\begin{equation}
	\forall g \in \G, s_1 \in \Sol_1, s_2 \in \Sol_2: Pr[|\hat L_{{1|s_2}_T} (g, s) - L_{{1|s_2}_{\cal D}} (g, s)| < \eps] > 1 - \delta
	\end{equation}
Therefore, $m^{UC} \in O(\max \lbrace m_i^{UC} \rbrace)$. 
\end{proof}
We note that the bound for conjuncts in Corollary~\ref{corr:conjuncts} trivially generalizes to any number of conjunctions, i.e. if $\lambda = \lambda_1\wedge \dots \wedge \lambda_q$, the dimension of $\lambda$ is upper-bounded by the dimension of the domains corresponding to $\lambda_1,\dots,\lambda_q$.
\corrargmax*
\begin{proof}
	Let us assume that a set $C = \lbrace x_1, x_2 \rbrace$ is S-shattered with $g \in \G$. This implies the existence of: i) $x'$ such that $\lambda(x_1, g, x') = 0 \implies g(x') \succeq g(x_1)$ and $\lambda(x_2, g, x') = 1 \implies g(x_2) \succ g(x')$. And, ii) $x''$ such that $\lambda(x_1, g, x'') = 1 \implies g(x_1) \succ g(x'')$ and $\lambda(x_2, g, x') = 0 \implies g(x'') \succeq 	g(x_2)$. However, since $\succeq$ is transitive, this leads to contradiction.
\end{proof}
Finally, we present the full proof for ERM solvers for non-realizable agnostic solution learning.
\corrpacsolvingERMbounds*
\begin{proof}
We first prove the result for worst-case agnostic learning. Let $\cal A_m : (\cal X \times \cal Y)^m \rightarrow \Sol$ be an ERM Solver that for any sample of $m \geq m^{\UC}(\eps, \delta)$ points $T = \langle (x_i, y_i) \rangle_{i=1}^m$, outputs a solution $\cal A_m (T) \in \Sol$ that minimizes $\underset{g \in \G|_{T}}{\max} \hat L_T(g, \cal A_m (T) )$. 

Let $s^* \in \Sol$ be a solution that minimizes the worst-case statistical loss for any game $g$ consistent with the sample, i.e. 
$$s^*\in \argmin_{s \in \Sol}\underset{g \in \G|_{T}}{\max} L_{\cal D}(g, s ).$$ 
By definition of $\cal A_m (T)$, we know that 
$$\underset{g \in \G|_{T}}{\max} \hat L_T(g, \cal A_m (T) ) \leq \underset{g \in \G|_{T}}{\max} \hat L_T(g, s^* ).$$ 

For a sample of $m  \geq m^{\UC}(\eps /2, \delta/2)$ points drawn i.i.d. from $\cal D$, we know by Theorem \ref{thm:upperbound}, that for any $g_0 \in \G$, with probability $\geq 1 - \delta/2$ we have: 
$$|L_{\cal D}(g_0, \cal A_m(T)) - \hat L_T(g_0, \cal A_m (T)) |< \frac{\eps}2,$$
and, with probability $\geq 1- \delta/2$, for any $g^\prime\in \G$, 
$$|\hat L_{T}(g_0, s^* ) -L_{\cal D}(g^\prime, s^* )| < \frac{\eps}{2}.$$

Putting it all together, we get that with probability $\geq \delta$,
\begin{align*}
L_{\cal D}(g_0, \cal A_m(T) ) \leq &\hat L_T(g_0, \cal A_m(T) ) + \frac{\eps}2 \leq \underset{g \in \G|_{T}}{\max} \hat L_T(g, \cal A_m (T) ) + \frac{\eps}{2}\\
	\leq &\underset{g \in \G|_{T}}{\max} \hat L_T(g, s^* ) + \frac{\eps}{2} \leq \underset{g \in \G|_{T}}{\max} L_{\cal D}(g, s^* ) + \frac{\eps}{2} + \frac{\eps}{2} = \underset{g \in \G|_{T}}{\max} L_{\cal D}(g, s^* ) + \eps.
\end{align*}

Therefore, if $m \geq m^{\UC}(\eps/2, \delta/2)$, then 
$$\underset{T \sim {\cal D}^m}{\Pr}[L_{\cal D}(g,\cal A_m(T)) \leq \underset{s\in \Sol}{\min} \underset{g^\prime \in \G|_{T}}{\max} L_{\cal D}(g^\prime, s) + \eps] \geq 1 - \delta.$$

For Bayesian agnostic learning, we present the result for distributions with a finite support over $\G$; the case where the distribution has an infinite support over $\G$ is similar. Let $\tilde{\cal D}$ be some prior distribution with a finite support over $\G$. Then, by definition
$${\E}_{g^\prime \sim \tilde{\cal D}}[L_{\cal D}(g^\prime,\cal A(T))| g^\prime \in \G|_{T}] = \sum_{g \in \G_{|T}} L_{\cal D}(g, s) \Pr_{\tilde{\cal D}}(g) .$$ 
When $m \geq m^{\UC}(\eps, \delta)$, by Theorem \ref{thm:upperbound}, for every $g \in \G_{|T}$ and $s\in \Sol$, $\Pr[|L_{\cal D}(g,s) -  \hat L_{T}(g,s) | \leq \eps] > 1- \delta$. 
Combining these expressions, with probability mass $\tilde{\cal D}$, we get 
$$\Pr[|{\E}_{g^\prime \sim \tilde{\cal D}}[L_{\cal D}(g^\prime,s)| g^\prime \in \G|_{T}] -  {\E}_{g^\prime \sim \tilde{\cal D}}[\hat L_{T}(g^\prime,s)| g^\prime \in \G|_{T}] | \leq \eps] > 1- \delta.$$ 
Therefore, an ERM solver that minimizes ${\E}_{g^\prime \sim \tilde{\cal D}}[\hat L_{T}(g^\prime,s)| g^\prime \in \G|_{T}]$, also bounds the statistical loss within $\eps$.
\end{proof}

%% file: apdxcoopgames.tex
\section{The PAC Core in TU Cooperative Games}\label{apdx:coopgames}
\subsection{Cooperative Games}\label{sec:coopgames-primer}
In {\em transferable utility} (TU) cooperative games players' preferences are induded by a function $v:2^N \to \R_+$ mapping every subset $S \subseteq N$ to a value $v(S) \in \R_+$. We are interested in finding ``good'' payoff divisions for the game. These are simply vectors $\vec x = (x_1,\dots,x_n) \in \R_+^n$ such that $\sum_{i = 1}^n x_i = v(N)$ (efficiency) and $x_i \ge v(\{i\})$ for all $i \in N$ (individual rationality). 
We say that a coalition $S \subseteq N$ {\em blocks} a payoff division $\vec x$ if $\sum_{i \in S} x_i < v(S)$; that is, the coalition $S$ can guarantee its members a strictly higher reward should they choose to break off from working with everyone else. 
The {\em core} is the (possibly empty) set of payoff divisions from which no coalition can deviate; in other words, $\core(N,v) = \{\vec x \in \R_+^n \mid \forall S \subseteq N: \sum_{i \in S} x_i \ge v(S); \sum_{i = 1}^n x_i = v(N)\}$. 
\subsection{The PAC Core for TU Cooperative Games}\label{sec:core-TU}
\citet{balcan2015learning} propose a learning-based approach to finding a PAC stable payoff division for TU cooperative games (see definitions in Section~\ref{sec:coopgames-primer}). 
Given a distribution $\cal D$ over $2^N$, a payoff division $\vec x^*$ {\em $\eps$-PAC stabilizes} the game $\tup{N,v}$ with respect to $\cal D$ if 
\begin{align*}
\Pr_{S \sim \cal D}\left[v(S) < \sum_{i\in S}x_i^*\right]<\eps.
\end{align*}
In what follows, we provide a proof for the PAC stabilizability of TU cooperative games in the language of Theorem~\ref{thm:upperbound}; direct proofs of this fact appear in \cite{balcan2015learning-arxiv,balkanski2017cost}.
\begin{theorem}
	The solution dimension of TU cooperative games is $\cal O(n)$.
\end{theorem}
\begin{proof}
We show that any set of $>n$ coalitions cannot be S-shattered as per Definition~\ref{def:S-shattering}. Taking a set of coalitions $\cal S = \{S_1,\dots,S_m\}$, it is S-shattered if there is some TU cooperative game $v:\cal S \to \R_+$ such that for all $\cal T \subseteq \cal S$, there exists some vector $\vec x^*$ in $\R^n$ such that for all $T \in \cal T$, $v(T) \ge x^(T)$, and for all $S \in \cal S\setminus \cal T$, $v(S) < x^*(S)$. Let us bound the dimension $m$ of $\cal S$. The problem is equivalent to shattering sets of vectors in the hypercube $\{0,1\}^n$ with linear classifiers, which is well-known to be impossible for sets of size $> n$ \cite{anthony1999learning}. We conclude that $\Sd$ for the PAC core of TU cooperative games is $\leq n$.  
\end{proof}
We note that the solution computed in \citet{balcan2015learning-arxiv} is only efficient (i.e. with $\sum_{i = 1}^n x_i = v(N)$) if the core of the cooperative game $v$ is not empty. In the case where the game $v$ has an empty core, the solution computed still satisfies the core constraints with high probability with respect to $\cal D$, but may not be efficient. However, the payoff outputted is using the {\em minimal subsidy required} in order to stabilize the game. In other words, the total payoff is no more than the {\em cost of stability} of the underlying game $v$ \cite{bachrach2018costab}. Efficiency is an important requirement: without it, one can ``cheat'' and pay each player some arbitrarily high amount, guaranteeing that the underlying game is stable. 

%% file: apdxmarkets.tex
\section{PAC Competitive Equilibria in Exchange Economies}\label{apdx:markets}
In Section~\ref{sec:markets-primer} we define Fisher markets; these are markets where goods are indivisible, and each player $i \in N$ has a budget $\beta_i$. In what follows, we consider {\em exchange economies}  \cite[Chapters 6 and 9]{agtbook}, which follow a somewhat different structure.
\subsection{Exchange Economies}
In exchange economies we have a set $G = \{g_1,\dots,g_k\}$ of $k$ {\em divisible} goods, and player valuations are of the form $v_i:[0,1]^k \to \R_+$ for every $i \in N$; bundle assignments are $\pi : N \to [0,1]^k$ (assigning a quantity $q_j\le 1$ of good $g_j$ to player $i$ can be thought of as player $i$ receiving $q_j$ percent of good $g_j$). 

In exchange economies with divisible goods, we assume that each player has an initial {\em endowment} of goods $\vec e_i\in [0,1]^k$, denoting the (divisible) amount of each good that she possesses. It is no loss of generality to assume that $\sum_{i = 1}^n \vec e_{i,j} = 1$ for every good $g_j$; in other words, the quantity $e_{i,j}$ is the relative amount of good $g_j$ that player $i$ possesses.
Given item prices, players {\em demand} certain item bundles.
The {\em affordable set} is the set of all divisible goods whose total price is less than the worth of player $i$'s endowment under $\vec p$.
$$\cal A_i(\vec p) = \left\{\vec g\in [0,1]^k: \sum_{j =1}^k p_jg_j \le \sum_{j = 1}^k p_je_{i,j} \right\}.$$
An outcome $\tup{\pi,\vec p}$ is a competitive equilibrium if $\pi(i) \in \cal A_i(\vec p)$, and $\forall \vec g \in \cal A(\vec p)$, $v_i(\pi(i)) \ge v_i(\vec g)$.
\subsection{PAC Market Equilibria in Exchange Economies}
We assume that player preferences are convex. 
We show that for any sample of fractional bundles $T = \{\vec b_1,\dots,\vec b_m \}$,
there exists a solution $\tup{\pi^*,\vec p^*}$ consistent with $T$ with non-positive excess assignment (but potentially leaving some goods unassigned). We assume that none of the goods are undesirable, i.e. for every good there exists at least one player that assigns a positive value to some quantity of that good.

\begin{theorem}\label{thm:exchange-divisible}
Suppose we are given an exchange economy for divisible goods with convex preferences and without undesirable goods. We observe $m$ sampled bundles $T = \{\vec b_1, \dots,\vec b_m\}$ and player valuations over the bundles, along with player endowments $\vec e_1,\dots,\vec e_n$. There exists a solution $\tup{\pi^*,\vec p^*}$ such that every player $i$ is assigned a bundle they can afford given their endowment, which is consistent (against any possible valuation functions that could have generated the observed values). 
\end{theorem}
\begin{proof}
Without loss of generality, let us work with the reduced space of only observed goods. Let $\cal U$ denote the underlying space of convex preferences from which we draw player preferences over assignments. Let $\cal U_{|i,T}$ denote the space of all valuation functions $u$ that satisfy the observed values, ie. $u(\vec b_j) = v_i(\vec b_j)$ for every $j \in [m]$. 
Since prices only need to satisfy the affordability criterion for every player, i.e. $\vec p^* \cdot \pi(i) \leq \vec p^* \cdot \vec e_i$, we can normalize and assume that prices belong to the simplex $\Delta_{n-1}$. Also, observe that the absence of undesirable goods implies that in any consistent solution the price of any observed good cannot be $0$.

Now let us define the demand set function as $D: \cal U \times \Delta_{n-1} \times [0,1]^n \rightarrow 2^{[0,1]^n}$, such that 
\begin{align*}
D(u, \vec p, \vec e_i) = \left\{ \vec b \in [0,1]^n : u(\vec b) \geq u(\vec b_j) \forall b_j \in T; \text{ and } \vec p \cdot \vec b \le \vec p \cdot \vec e_i\right\}.
\end{align*} 
We observe that under convex preferences (i.e. quasi-concave utility functions), for every $u$, $\vec p$ and $\vec e$, $D(u,\vec p,\vec e)$ is a convex and compact body; in addition, $D$ is continuous in $\vec p$. 
Define, for every player $i\in N$, $D_i^T(\vec p) = \bigcap_{u \in \cal U|_{i,T}} D(u, \vec p, \vec e)$: $D_i^T(\vec p)$ is the set of all possible bundles that player $i$ might demand under the price vector $\vec p$, under all possible utility functions that agree with the sample $T$. The intersection $D_i^T$ is convex and compact, as well as continuous in $\vec p$. Also observe that $D_i^T(\vec p)$ is always non-empty, since there is at least one bundle among the observed samples and the empty bundle which belongs to each of the $D(u, \vec p, \vec e)$.

Let $f: [0,1]^{k \times n} \to [0,n]^k$ be the {\em excess demand function}: $f(\pi) = \sum_{i \in N} \pi(i) - \vec 1$ (where $\vec{1} = (1,1,\dots,1) \in [0,1]^n$). 
Let $\mathbf{z} : \Delta_{n-1} \rightarrow 2^{[0,k]^n}$, be the function $\mathbf{z}(\vec p) = \lbrace f(\pi) : \pi \in \prod_{i \in K} D_i^T (\vec p) \rbrace$. 
The function $f$ is linear, therefore $\mathbf{z}(\vec p)$ is convex and compact, and $\mathbf{z}$ is continuous. 
Using $\mathbf{z}$, we define a function $\mathbf{g} : \Delta_{n-1} \to 2^{\Delta_{n-1}}$ 
such that $r$-th component is given by 
\begin{align*}
\mathbf{g}(\vec p) = \left\{ {\vec g}: \mbox{ where } g_r(\vec p) = \frac{p_r + \max\{0, z_r\}}{1 + \sum_{s=1}^n \max\{0,z_s\}},\mbox{ for some } \vec z \in \mathbf{z}(\vec p) \right\}
\end{align*}


By applying Kakutani's fixed-point theorem over $\mathbf{g}$, we get the existence of some $\vec p^*$ such that $\vec p^* \in \mathbf{g}(\vec p^*)$. This implies the existence of some $\pi \in \prod_{i \in K} D_i^T (\vec p)$, such that $f(\pi) \in \mathbf{z}(\vec p^*)$ satisfies 
$$p^*_r = \frac{p^*_r + \max(0, f_r(\vec p^*))}{1 + \sum_{s=1}^n \max(0, f_s(\vec p^*))}$$
Let $r^*$ some non-positive component of $f(\pi)$, as argued above; then $p^*_r = p^*_r (1 + \sum_{s=1}^n \max(0, f_s(\vec p^*)))$; this implies that for all $r$: $\max(0,f_r(\vec p^*)) = 0$. Therefore, there exists some allocation $\pi^*$ with non-positive excess demand at $\vec p^*$, such that $\tup{\pi^*,\vec p^*}$ is consistent with the observed bundles against all possible $u \in \cal U|_{i,T}$ for every player $i$.
\end{proof}

%% file: apdxcondorcet.tex
\section{PAC Condorcet Winners}\label{apdx:condorcet}
Recall that a Condorcet winner is a candidate $c^*$ that beats every other candidate $c$ in a pairwise election, i.e. for every other candidate $c$, a majority of voters prefer $c^*$ to $c$. We note that if there are only two candidates, then a Condorcet winner trivially exists (barring the case when the votes are tied); however, when there are three or more candidates, a Condorcet winner is not guaranteed to exist. 
This case is known in the literature as Condorcet cycles (or Condorcet paradoxes). For the sake of completeness, we provide a simple example of a voting profile where no candidate is a Condorcet winner.
\begin{example}\label{ex:condorcet-cycle}
	Consider a setting with three candidates $a,b,c$ and three voters, $1,2$ and $3$ whose preferences over $a,b,c$ are as follows:
	\begin{align*}
	1:& a \succ_1 b \succ_1 c\\
	2:& b \succ_2 c \succ_2 a\\
	3:& c \succ_3 a \succ_3 b	
	\end{align*}
	1 and 3 prefer $a$ to $b$; 1 and 2 prefer $b$ to $c$; 2 and 3 prefer $c$ to $a$. Therefore, there are no Condorcet winners.
	\end{example}
Next, let us provide a complete proof of Theorem~\ref{thm:no-condorcet}. 
\thmnocondorcet*
\begin{proof}
Let us assume that there is a set of size $d$, $C_0 \subseteq C$, that is shattered. Then by definition of shattering, there exists $2^d$ different candidates corresponding to every subset of $C_0$, such that for $f \in 2^{C_0}$ there exists a candidate $c_f \in C$, such that $c_1 \in C_0$ beats $c_f$ in the tournament graph if and only if $f(c_1) = 1$.

Let us focus on $2^d - 2$ of these candidates, corresponding to all non-trivial functions $f \in 2^{C_0}$ (let us, for now, ignore the functions that assign a constant value (of $0$ or $1$) to all candidates in $C_0$). Then, for every pair of these functions $f_1$ and $f_2$, there exists a candidate $c_1 \in C_0$ such that $f_1(c_1) = 1$ and $f_2(c_1) = 0$, and a candidate $c_2$ such that $f_1(c_2) = 0$ and $f_2(c_2) = 1$. This implies the existence of a directed path of length at most $2$ from $c_{f_1}$ to $c_{f_2}$, and vice versa. Since, in a tournament graph, either the edge $c_{f_1}\to c_{f_2}$ or $c_{f_2}\to c_{f_1}$ exists, we know that $c_{f_1}$ and $c_{f_2}$ are members of some $3$-cycle. Since this is true for all such $c_f$'s, we know that $2^d - 2$ is less than or equal to largest number of candidates, such that for some tournament graph in $\HC$, every pair amongst them is part of some $3$-cycle. 
\end{proof}